\documentclass[twoside]{article}
\usepackage{helvet}
\usepackage{courier}
\usepackage[latin9]{inputenc}
\usepackage{amsmath}
\usepackage{amssymb}
\usepackage{graphicx}
\usepackage{esint}

\makeatletter

\usepackage{aistats2014}

\usepackage{algorithmic}\usepackage{amsfonts}\usepackage{amsthm}\usepackage{multirow}\usepackage{color}\usepackage{cite}
\newtheorem{theorem}{Theorem}
\newtheorem{corollary}{Corollary}
\newtheorem{assumption}{Assumption}
\newtheorem{definition}{Definition}

\newtheorem{lemma}{Lemma}

\DeclareMathOperator*{\argmax}{arg\,max}

\allowdisplaybreaks 

\newcommand{\comment}[1]{}

\ifodd 1
\newcommand{\com}[1]{{\color{red}#1}}
\else
\newcommand{\com}[1]{}
\fi

\makeatother

\begin{document}


\twocolumn[

\aistatstitle{Global Multi-armed Bandits with Hölder Continuity}

\aistatsauthor{ Anonymous Author 1 \And Anonymous Author 2 \And Anonymous Author 3 }

\aistatsaddress{ Unknown Institution 1 \And Unknown Institution 2 \And Unknown Institution 3 } ] 
\begin{abstract}
\comment{ In Multi-armed bandit (MAB) problems, a learner maximizes
its gains by sequentially choosing among a set of arms with unknown
rewards. Choosing an arm reveals information about the arm's reward
at the expense of missing the opportunity of collecting the (possibly
higher) reward of another arm. In the classical MAB problem the arms
are assumed to be independent. Hence, choosing an arm reveals no information
about the rewards of other arms. Thus, identifying the best arm requires
all arms to be selected at least logarithmically many times, thereby
resulting in logarithmic regret. } Standard Multi-Armed Bandit (MAB) problems assume that the arms are independent. However, in many application scenarios, the information obtained by playing an arm provides information about the remainder of the arms. Hence, in such applications, this informativeness can and should be exploited to enable faster convergence to the optimal solution. In this paper, we introduce and formalize the Global MAB (GMAB), in which arms are \textit{globally} informative through a \emph{global parameter}, i.e., choosing an arm reveals information about \emph{all} the arms. We propose a greedy policy
for the GMAB which always selects the arm with the highest estimated
expected reward, and prove that it achieves \textit{bounded parameter-dependent
regret}. Hence, this policy selects
suboptimal arms only finitely many times, and after a finite number
of initial time steps, the \textit{optimal arm} is selected in \textit{all}
of the remaining time steps with probability one. In addition, we
also study how the \emph{informativeness} of the arms about each other's rewards affects
the speed of learning. Specifically, we prove that the parameter-free (worst-case) regret is sublinear in time, and
decreases with the informativeness of the arms. We also prove
a sublinear in time {\em Bayesian risk} bound for the GMAB which reduces to
the well-known Bayesian risk bound for linearly parameterized bandits
when the arms are \emph{fully informative}. GMABs have applications
ranging from drug and treatment discovery to dynamic pricing. 
\end{abstract}

\section{Introduction}

In this paper we study a new class of MAB problems which we name the Global MAB (GMAB). In the GMAB problem, a learner sequentially
selects one of the available $K$ arms with the goal of maximizing
its total expected reward. We assume that expected reward of arm $k$ is
$\mu_{k}(\theta_{*})$, where $\theta_{*}\in\Theta$ is an unknown global parameter. 
For the given global parameter $\theta_{*}$, the reward of each arm follows an i.i.d. process. 
The learner knows the expected reward function $\mu_{k}(\cdot)$ of
all the arms $k$. In this setting an arm $k$ is informative about
another arm $k'$ because the learner can estimate the expected reward
of arm $k'$ by using the estimated reward of arm $k$ and the expected
reward functions $\mu_{k}(\cdot)$ and $\mu_{k'}(\cdot)$. Under mild
assumptions on the expected reward functions, we prove that a greedy
policy which always selects the arm with the highest estimated expected
reward achieves \textit{bounded regret}, which is independent of time.
In other words, suboptimal arms are selected only finitely many times
before converging to the optimal arm. This is a surprising result,
since as shown in \cite{lairobbinsl}, it is not possible to achieve
bounded regret in standard MAB problems because playing arm $k$ is
the only option to learn about its expected reward in these problems.
\comment{ This is the case because the expected arm rewards in standard
MAB problems\cite{lairobbinsl} are independent of each other, in
contrast to the GMAB problem in which the expected rewards of the
arms are related through $\theta_{\textrm{true}}$. } 

While most of the literature on MAB problems assumes independent arms \cite{lairobbinsl,A2002,Kauffman}
and focuses on achieving regret that is logarithmic in time, structured MAB problems exist in which bounded regret has been proven.
One prominent example is provided in \cite{Tsiklis_structured} in
which the expected rewards of the arms are known linear functions
of a global parameter. Under this assumption, \cite{Tsiklis_structured}
proves that the greedy policy achieves bounded regret. Proving finite
regret bounds under this linearity assumption becomes possible since
all the arms are fully informative about each other, i.e., rewards obtained from an arm can be used to estimate the expected reward of the other arms using a linear transformation on the obtained rewards. 

In this paper we consider a more general model in which the expected reward functions are {\em Hölder continuous}, which requires to use a non-linear estimator to exploit the {\em weak informativeness}. Thus, our model includes the case when the expected reward functions are linear functions as a special case. However, while our regret results are a generalization of the
results in \cite{Tsiklis_structured}, our analysis of the regret
is more complicated since the arms are not fully informative as in
the linear case. Thus, deriving regret bounds in our setting requires
us to develop new proof techniques. However, we also show that our
learning algorithm and the regret bounds reduce to the ones in \cite{Tsiklis_structured}
when the arms have linear reward functions. 
In addition to
the bounded regret bound (which depends on the value of the parameter
$\theta_{*}$), we also provide a parameter-free regret
bound and a bound on the Bayesian risk given a distribution $f(\cdot)$
over the parameter space $\Theta$, which matches known upper bound $\Omega(\log T)$ for the linear reward functions \cite{Tsiklis_structured}. Both of these bounds are sublinear in time and depend on the informativeness of the arms with respect
to the other arms, subsequently referred to shortly as informativeness.

Many applications can be formalized as a GMAB, where the reward functions
are {\em Hölder continuous} in the global parameter. Examples include clinical
trials involving similar drugs (e.g., drugs with a similar chemical
composition) or treatments which may have similar effects on the patients
and hence, the outcome of administering one drug/treatment to a patient
will yield information about the outcome of administering a similar
drug/treatment to that patient. Another example is dynamic pricing
\cite{dynamic}. In dynamic pricing, an agent sequentially selects
a price from a set of prices ${\cal P}$ with the objective of maximizing
its revenue over a finite time horizon. At time $t$, the agent first
selects a price $p\in{\cal P}$, and then observes the amount of sales,
which is given as $S_{p,t}(\Lambda)=\bar{F}_{p}(\Lambda)+\epsilon_{t}$, 
where $\bar{F}(.)$ is modulating function and $\epsilon_{t}$ is
the noise term with zero mean. The modulating function is the purchase
probability of an item of price $p$ given the market size $\Lambda$.
Here, the market size is the global parameter, which is unknown and
needs to be learned by setting any price and observing the sales related
to that price. Commonly used modulating functions include the exponential
and logistic functions.

In summary, the main contributions of our paper are: 
\vspace{-0.25in}
\begin{itemize}
\item We formalize a new class of structured MAB problems, which we refer
to as Global MABs. This class of problems represents a generalization
of the linearly parametrized bandits in \cite{Tsiklis_structured}. 
\item For GMABs, we propose a greedy policy that always selects the arm
with the highest estimated expected reward. We prove that the greedy
policy achieves bounded regret (independent of time horizon $T$,
depending on $\theta_{*}$). 
\item In addition to proving that the regret is bounded (which is related
to the asymptotic behavior), we also show how the regret increases
over time by identifying and characterizing three regimes of growth:
first, the regret increases at most sublinearly over time until a
first threshold (that depends on the informativeness) after which
it increases at most logarithmically over time until a second threshold,
before converging to a finite number asymptotically. These thresholds
have the property that they are decreasing in the informativeness. 
\item We prove a sublinear in time worst-case (parameter-free) regret bound.
The rate of increase in time decreases with the informativeness of
the arms, meaning that the regret will increase slower when the informativeness
is high. 
\item Given a distribution over the set of global parameter values, we prove
a Bayesian risk bound that depends on the informativeness. When the
arms are \emph{fully informative}, such as in the case of linearly
parametrized bandits \cite{Tsiklis_structured}, our Bayesian risk
bound and our proposed greedy policy reduce to the well known Bayesian
risk bound and the greedy policy in \cite{Tsiklis_structured}, respectively. 
\end{itemize}
\vspace{-0.2in}
\subsection{Related Work}
\vspace{-0.1in}
Numerous types of MAB problems have been defined and investigated
in the past decade - these include stochastic bandits \cite{lairobbinsl,A2002,Auer_2002a,KL_divergence,Agawal_89},
Bayesian bandits \cite{Kauffman,Thompson,Goyal_Thompson,KauffmanThompson,thompson_priorfree},
contextual bandits \cite{epoch_greedy,Slivkins,thompson_contextual},
combinatorial bandits \cite{combinatorial}, and many other variants.
Instead of comparing our method against all these MAB variants, we
group the existing literature based on the main theme of this paper:
exploiting the informativeness of an arm to learn about the rewards
of other arms. We call a MAB problem {\em non-informative} if the
reward observations of any arm do not reveal any information about
the expected rewards of any other arms. Examples of non-informative
MAB are the stochastic bandits \cite{lairobbinsl,A2002} and the bandits
with local parameters \cite{Goyal_Thompson,Kauffman}. In these problems
the regret grows at least logarithmically in time, since each arm
should be selected at least logarithmically many times to identify
the optimal arm. We call a MAB problem {\em group-informative}
if the reward observations from an arm provide information about the
rewards of a known group of other arms but not all the arms. Examples
of group-informative MAB problems are combinatorial bandits \cite{combinatorial},
contextual bandits \cite{epoch_greedy,Slivkins,thompson_contextual}
and structured bandits \cite{Tsiklis,paramtric_GLM}. In these problems
the regret grows at least logarithmically over time since at least
one suboptimal arm should be selected at least logarithmically many
times to identify groups of arms that are suboptimal. We call a MAB
problem {\em globally-informative} if the reward observations from an arm
provides information about the rewards of \textit{all} the arms. The
proposed GMABs include the linearly-parametrized MABs in \cite{Tsiklis_structured}
as a subclass. Therefore, we prove a bounded regret for a larger class of problems. 

Another related work is \cite{gittins}, in which the optimal arm
selection strategy is derived for the infinite time horizon learning
problem, when the arm rewards are parametrized with known priors,
and the future rewards are discounted. However, in the Gittins' formulation
of the MAB problem, the parameters of the arms are different from
each other, and the discounting allows the learner to efficiently
solve the optimization problem related to arm selection by decoupling
the joint optimization problem into $K$ individual optimization problems
- one for each arm. In contrast, we do not assume known priors, and
the learner in our case does not solve an optimization problem but
rather learns the global parameter through its reward observations.

Another seemingly related learning scenario is the experts setting \cite{Experts}, where after
an arm is chosen, the rewards of all arms are observed and their estimated
rewards is updated. Hence, there is no tradeoff between exploration
and exploitation and finite regret bounds can  be achieved in an expert
system with finite number of arms and stochastic arm rewards. However,
unlike in the expert setting, the GMABs achieve finite regret bounds
while observing \textit{only} the reward of the selected arm. Hence,
the arm reward estimation procedure in GMABs requires forming reward
estimates by collectively considering the observed rewards from all
the arms, which is completely different than in the expert systems,
in which the expected reward of an arm is estimated only by using
the past reward observations from that arm.
\vspace{-0.15in}
\section{Global Multi-Armed Bandits}
\vspace{-0.1in}
\subsection{Problem Formulation}
\vspace{-0.1in}
The set of all arms is denoted by ${\cal K}$ and the number of arms
is $K=|{\cal K}|$, where $|\cdot|$ is the cardinality operator. The
reward obtained by playing an arm $k\in{\cal K}$ at time $t$ is
given by a random variable $X_{k,t}$. We assume that for $t\geq1$
and $k\in{\cal K}$, $X_{k,t}$ is drawn independently from an unknown
distribution $\nu_{k}(\theta_{*})$ with support $[0,1]$.\footnote{The set $[0,1]$ is just a convenient normalization. In general, we only need that distribution has a bounded support.} The learner knows that the expected reward of an arm
$k\in{\cal K}$ is a (Hölder continuous, invertible) function of
the global parameter $\theta_{*}$, which is given by $E_{X_{k,t}\sim\nu_{k}(\theta_{*})}(X_{k,t})=\mu_{k}(\theta_{*})$, where $\mu_{k}: \Theta \rightarrow [0,1]$ and $E[\cdot]$ denotes the expectation. Hence, the true expected reward of arm $k$ is equal to $\mu_{k}(\theta_{*})$.


\begin{assumption} \label{ass:holder}

(i) For each $k\in{\cal K}$, the reward function $\mu_{k}$ is invertible
on $[0,1]$. \\
 (ii) For each $k\in{\cal K}$ and $y,y'\in[0,1]$, there exists $D_{1}>0$
and $0<\gamma_{1}\leq1$ such that $|\mu_{k}^{-1}(y)-\mu_{k}^{-1}(y')|\leq D_{1}|y-y'|^{\gamma_{1}}$, where $\mu_{k}^{-1}$ is the inverse reward function for arm $k$.
\\
 (iii) For each $k\in{\cal K}$ and $\theta,\theta '\in \Theta$ there
exists $D_{2}>0$ and $0<\gamma_{2}\leq1$, such that $|\mu_{k}(\theta)-\mu_{k}(\theta')|\leq D_{2}|\theta-\theta'|^{\gamma_{2}}$.

\end{assumption}

\vspace{-0.05in}

Assumption \ref{ass:holder} ensures that the reward obtained
from an arm can be used to update the estimated expected rewards of
the other arms. The last two conditions are Hölder conditions on the reward and inverse reward functions, which enable us to define the informativeness. It turns out that the invertibility of the reward functions is a crucial assumption that is required to achieve bounded regret. We illustrate this by a counter example when we discuss parameter dependent regret bounds. 
\comment{These assumptions are not necessary for
our algorithm to run, but are required to derive the bounded regret.}

There are many reward functions that satisfy Assumption \ref{ass:holder}.
Examples include: ($i$) exponential functions such as $\mu_{k}(\theta)=a\exp(b\theta)$
for some $a>0$, ($ii$) linear and piecewise linear functions, and
($iii$) sub-linear and super-linear functions in $\theta$ which
are invertible in $\Theta$ such as $\mu_{k}(\theta)=a\theta^{\gamma}$
with $\gamma>0$.

The goal of the learner is to choose a sequence of arms (one at each
time) $\boldsymbol{I}:=(I_{1},\ldots,I_{T})$ up to to time $T$ to
maximize its expected total reward. This corresponds to minimizing
the regret which is the expected total loss due to not always selecting
the optimal arm, i.e., the arm with the highest expected reward. Let
$k^{*}(\theta_{*}):=\argmax_{k\in{\cal K}}\mu_{k}(\theta_{*})$ be the set
of optimal arms and $\mu^{*}(\theta_{*}):=\max_{k\in{\cal K}}\mu_{k}(\theta_{*})$
be the expected reward of the optimal arm for true value of global parameter $\theta_{*}$. The cumulative regret of learning algorithm which selects arm $I_{t}$ until time horizon $T$ is defined as 
\vspace{-0.22in}
\begin{align}
\text{Reg}(\theta_{*},T):=\sum_{t=1}^{T} r_{t}(\theta_{*}),
\end{align}
\vspace{-0.25in}

where $r_{t}(\theta_{*})$ is the one step regret given
by $r_{t}(\theta_{*}):=\mu^{*}(\theta_{*})-\mu_{I_{t}}(\theta_{*})$ for global parameter $\theta_{*}$.
In the following
sections we will derive regret bounds both as a function of $\theta_{*}$
(parameter-dependent regret) and independent from $\theta_{*}$
(worst-case or parameter-free regret).

\vspace{-0.1in}
\subsection{Greedy Policy}
\vspace{-0.1in}
\begin{figure}[h!]
\fbox{%
\begin{minipage}[c]{0.95\columnwidth}%
{\fontsize{9}{8}\selectfont \begin{algorithmic}

\STATE{Input : $\mu_{k}$ for each $k\in{\cal K}$.} \STATE{Initialization:
$w_{k}(0)=0,\hat{\theta}_{k,1}=0,N_{k}(0)=0$ for all $k\in{\cal K}$.}
\WHILE{$t\geq1$}
 \IF{$t=1$} 
 \STATE{Randomly select arm $I_{t}$ from the set ${\cal K}$} 
\ELSE 
\STATE{Select the arm $I_{t}\in\argmax_{k\in{\cal K}}\mu_{k}(\hat{\theta}_{t-1})$}
\ENDIF
\STATE{Observe the reward $X_{I_{t},t}$} 
\STATE{$\hat{X}_{k,t}=\hat{X}_{k,t-1}$ for all $k\in{\cal K}\setminus I_{t}$} 
\STATE{$\hat{X}_{I_t,t} = \frac{\hat{X}_{I_t,t} + X_{I_{t},t}}{N_k(t) +1}$}
\STATE{Update individual estimates for global parameter as $\hat{\theta}_{k,t}=\mu_{k}^{-1}(\hat{X}_{k,t})$
for all $k\in{\cal K}$} 
\STATE{Update counters $N_{I_{t}}(t)=N_{I_{t}}(t-1)+1$} 
\STATE{Update the rest $N_{k}(t)=N_{k}(t-1)$ for all $k\in{\cal K}\setminus I_{t}$}
\STATE{Update weights $w_{k}(t)=\frac{N_{k}(t)}{t}$ for all $k\in{\cal K}$}
\STATE{$\hat{\theta}_{t}=\sum_{k=1}^{K}w_{k}(t)\hat{\theta}_{k,t}$}
\ENDWHILE \end{algorithmic} }%
\end{minipage}} \protect\caption{Pseudocode of the greedy policy.}

\vspace{-0.20in}
 \label{fig:GP} 
\end{figure}
In this section, we propose a greedy policy for the GMAB
problem, which selects the arm with the highest estimated expected
reward at each time $t$. Different from previous works in MABs \cite{A2002,lairobbinsl}
in which the expected reward estimate of an arm only depends on the
reward observations from that arm, the proposed greedy policy uses
a {\em global parameter estimate} $\hat{\theta}_{t}$ for the global
parameter, which is given by  $\hat{\theta}_{t}:=\sum_{k=1}^{K}w_{k}(t)\hat{\theta}_{k,t}$, where $w_{k}(t)$ is the weight of arm $k$ at time $t$ and $\hat{\theta}_{k,t}$ is the estimate of the global parameter based only on the reward observations
from arm $k$ until time $t$. Let ${\cal X}_{k,t}$ denote the set
of rewards obtained from the selections of arm $k$ by time $t$,
i.e., ${\cal X}_{k,t}=(X_{k,t})_{\tau<t\;|I_{\tau}=k}$ and $\hat{X}_{k,t}$
be the sample mean estimate of the rewards obtained from arm $k$
by time $t$, i.e., ${\hat{X}_{k,t}}:=(\sum_{x\in{\cal X}_{k,t}}x)/|{\cal X}_{k,t}|$.
The proposed greedy policy operates as follows for any time $t \geq 2$: ($i$) the arm with
highest expected reward according to the estimated parameter $\hat{\theta}_{t-1}$
is selected, i.e., $I_{t}\in\argmax_{k\in{\cal K}}\mu_{k}(\hat{\theta}_{t-1})$,
($ii$) reward $X_{I_{t},t}$ is obtained and individual reward estimates
$\hat{X}_{k,t}$ are updated for $k\in{\cal K}$, ($iii$) the individual
estimates of each arm $k$ for the global parameters are updated as
$\hat{\theta}_{k,t}=\mu_{k}^{-1}({\hat{X}_{k,t}})$, ($iv$) the weights
of each arm $k$ are updated as $w_{k}(t)=N_{k}(t)/(t)$, where $N_{k}(t)$
is the number of times the arm $k$ is played until time $t$, i.e.,
$N_{k}(t)=|{\cal X}_{k,t}|$. For $t =1$, since there is no global parameter estimate, the greedy policy selects randomly among the set of arms. The pseudocode of the greedy policy is given in Fig. \ref{fig:GP}.
\vspace{-0.1in}
\section{Regret Analysis for the Greedy Policy}
\vspace{-0.1in}
\subsection{Preliminaries}
\vspace{-0.1in}
In this subsection we define the tools that will be used in deriving
the regret bounds. Consider any arm $k\in{\cal K}$. Its \emph{optimality
region} is defined as $\Theta_{k}:=\{\theta\in\Theta\;|k\in k^{*}(\theta)\}$.
\comment{Since the reward functions are known a priori, the optimality
regions of the arms can be calculated a priori.} Clearly, we have
$\bigcup_{k\in{\cal K}}\Theta_{k}=\Theta$. If $\Theta_{k}=\emptyset$
for an arm $k$, this implies that there exists no global parameter
values for which arm $k$ is optimal. Since there exists an arm $k'$ such that $\mu_{k'}(\theta) >\mu_k(\theta)$ for any $\theta \in \Theta$ for an arm with $\Theta_k = \emptyset$, the greedy policy will discard arm $k$ after $t= 1$. Therefore, without loss of generality we assume that $\Theta_{k}\neq\emptyset$
for all $k\in{\cal K}.$ For global parameter $\theta_{*} \in \Theta$, we define the \emph{suboptimality gap }of an arm
$k\in{\cal K} \setminus k^{*}(\theta_{*})$ as $\delta_{k}(\theta_{*}):=\mu^{*}(\theta_{*})-\mu_{k}(\theta_{*})$.
For parameter $\theta_{*}$, the minimum suboptimality gap is defined as
$\delta_{\text{min}}(\theta_{*}):=\min_{k\in{\cal K}\setminus k^{*}(\theta_{*})}\delta_{k}(\theta_{*})$.
\comment{ In our analysis, we will show that when the expected reward
estimate for an arm $k$ is within $\delta_{\min}(\theta_{\textrm{true}})/2$
of its true expected reward for all arms $k\in{\cal K}$, then the
greedy policy will select the optimal arm, even when its global parameter
estimate is not exactly equal to $\theta_{\textrm{true}}$. }

\begin{figure}[h]
\vspace{-0.13in}

 \includegraphics[clip,width=1.0\columnwidth,trim = 50mm 0 30mm 30mm]{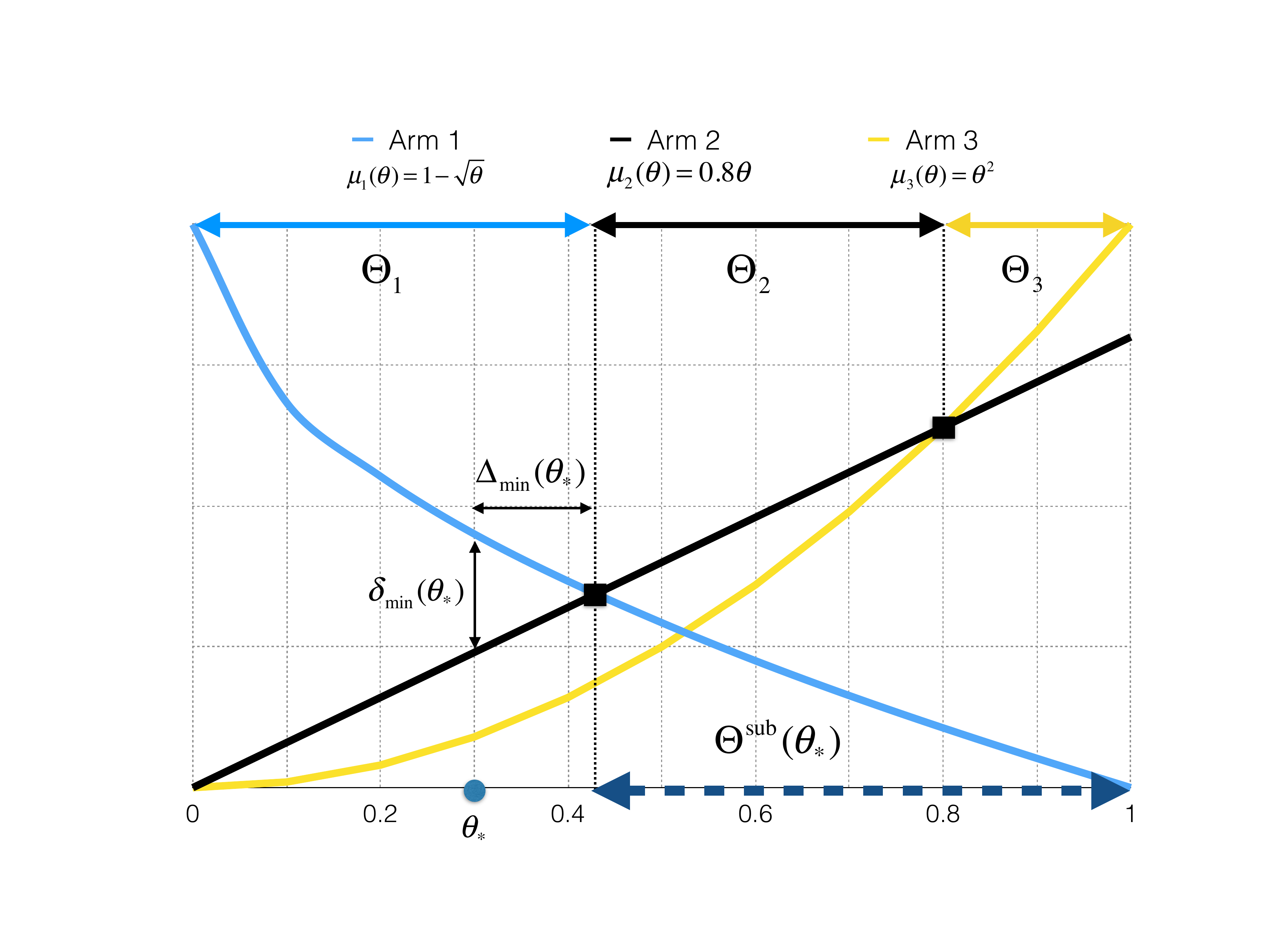}
 \vspace{-0.6in}
 \protect\caption{Illustration of minimum suboptimality gap and suboptimality distance}
\label{fig: illustration} 
\end{figure}

\vspace{-0.1in}

Recall that the expected reward estimate for arm $k$ is equal to
its expected reward corresponding to the global parameter estimate.
We will show that as more arms are selected, the global parameter
estimate will converge to the true value of the global parameter.
However, if $\theta_{*}$ lies close to the boundary of
the optimality region of $k^{*}(\theta_{*})$, the global
parameter estimate may fall outside of the optimality region of $k^{*}(\theta_{*})$
for a large number of time steps, thereby resulting in a large regret.
Let $\Theta^{\text{sub}}(\theta_{*})$ be the suboptimality region for given global parameter $\theta_{*}$, which is defined as the subset of parameter space in which an arm in the set ${\cal K} \setminus k^{*}(\theta_{*})$ is optimal, i.e $\Theta^{\text{sub}}(\theta_{*}) = \cup_{k' \in {\cal K} \setminus k^{*}(\theta_{*})} \Theta_{k'}$. In order to bound the expected number of such deviations  from the
optimality region, for any arm $k$ we define a metric called the
\emph{suboptimality distance, }which is equal to the smallest distance
between the value of the global parameter and suboptimality region.

\begin{definition} \label{def:dist} For a given global parameter
$\theta_{*}$, the \textit{suboptimality distance} is defined as 
\[
\Delta_{\text{min}}(\theta_{*}):=\left\{ \begin{array}{lr}
\inf_{\theta'\in\Theta^{\text{sub}}(\theta_{*})} |\theta_{*}-\theta'| & \text{if } \Theta^{\text{sub}}(\theta_{*}) \neq \emptyset\\
1 & \text{if }\Theta^{\text{sub}}(\theta_{*}) = \emptyset
\end{array}\right.
\]
\end{definition}

\vspace{-0.1in}

From the definition of the suboptimality distance it is evident that
the greedy policy always selects an optimal arm in $k^{*}(\theta_{*})$
when $\hat{\theta}_{t}$ is within $\Delta_{\text{min}}(\theta_{*})$
of the global parameter $\theta_{*}$. An illustration of suboptimality gap and suboptimality distance is given in Fig. \ref{fig: illustration}
for a GMAB problem instance with $3$ arms and reward functions $\mu_{1}(\theta)=1-\sqrt{\theta}$,
$\mu_{2}(\theta)=0.8\theta$ and $\mu_{3}(\theta)=\theta^{2}$.

\comment{
In the following lemma we characterize the minimum amount of change
in the value of the global parameter that will make a suboptimal action
an optimal action. This change is given in terms of $\delta_{min}(\theta)$,
and will be used to bound the regret by bounding the probability of
the deviation of $\hat{\theta}_{t}$ from $\theta$
by an amount greater than $\Delta_{min}(\theta)$.
}
In the following lemma, we show that minimum suboptimality distance is nonzero for any global parameter $\theta_{*}$. This result ensures that we can identify the optimal arm within finite amount of time. 

\begin{lemma} \label{prop:non_zero} Given any $\theta_{*}\in\Theta$,
there exists a constant $\epsilon_{\theta_{*}}=\delta_{\min}(\theta_{*})^{1/\gamma_{2}}/(2D_{2})^{1/\gamma_{2}}$,
where $D_{2}$ and $\gamma_{2}$ are the constants given in Assumption
1 such that $\Delta_{\min}(\theta_{*})\geq\epsilon_{\theta_{*}}.$ In other
words, the minimum suboptimality distance is always positive. 
\end{lemma}

\comment{
\begin{proof} For any suboptimal arm $k\in{\cal K}-k^{*}(\theta)$,
we have $\mu_{k^{*}(\theta)}(\theta)-\mu_{k}(\theta)\geq\delta_{\min}(\theta)>0.$
We also know that $\mu_{k}(\theta')\geq\mu_{k^{*}(\theta)}(\theta')$
for all $\theta'\in\Theta_{k}$. Hence for any $\theta'\in\Theta_{k}$
at least one of the following should hold: (i) $\mu_{k}(\theta')\geq\mu_{k}(\theta)-\delta_{\min}(\theta)/2$,
(ii) $\mu_{k^{*}(\theta)}(\theta')\leq\mu_{k^{*}(\theta)}(\theta)+\delta_{\min}(\theta)/2$.
If both of the below does not hold, then we must have $\mu_{k}(\theta')<\mu_{k^{*}(\theta)}(\theta')$,
which is false. This implies that we either have $\mu_{k}(\theta)-\mu_{k}(\theta')\leq\delta_{\min}(\theta)/2$
or $\mu_{k^{*}(\theta)}(\theta)-\mu_{k^{*}(\theta)}(\theta')\geq - \delta_{\min}(\theta)/2$,
or both. Recall that from Assumption 1 we have $|\theta-\theta'|\geq|\mu_{k}(\theta)-\mu_{k}(\theta')|^{1/\gamma_{2}}/D_{2}^{1/\gamma_{2}}$.
This implies that $|\theta-\theta'|\geq\epsilon_{\theta}$ for all
$\theta'\in\Theta_{k}$. \end{proof}}
For notational brevity, we denote in the remainder of the paper $\Delta_{\text{min}}(\theta_{*})$
and $\delta_{\text{min}}(\theta_{*})$ as $\Delta_{*}$ and $\delta_{*}$,
respectively.

\begin{lemma} \label{lemma:gap} Consider a run of the greedy policy
until time $t$. Then, the following relation between $\hat{\theta}_{t}$
and $\theta_{*}$ holds with probability one: $|\hat{\theta}_{t}-\theta_{*}|\leq\sum_{k=1}^{K}w_{k}(t)D_{1}|\hat{X}_{k,t}-\mu_{k}(\theta_{*})|^{\gamma_{1}}$
\end{lemma}

Lemma \ref{lemma:gap} shows that the gap between the global parameter
estimate and the true value of the global parameter is bounded by
a weighted sum of the gaps between the estimated expected rewards
and the true expected rewards of the arms.

\begin{lemma} \label{lemma:onesteploss} For given global parameter $\theta_*$, the one step regret of the
greedy policy is bounded by $r_{t}(\theta_{*})=\mu^{*}(\theta_{*})-\mu_{I_{t}}(\theta_{*}) \leq 2D_{2}|\theta_{*}-\hat{\theta}_{t}|^{\gamma_{2}}$ with probability one, where $I_{t}$ is the arm selected by the greedy policy at time $t\geq2$. \end{lemma}

\vspace{-0.05in}

Lemma \ref{lemma:onesteploss} ensures that the one step loss decreases
as $\hat{\theta}_{t}$ approaches to $\theta_{*}$. Since
the regret at time $T$ is the sum of the one step losses up to time
$T$, we will bound the regret by bounding the expected distance between
$\hat{\theta}_{t}$ and $\theta_{*}$.

Given a parameter value $\theta_{*}$, let ${\cal G}_{\theta_{*},\hat{\theta}_{t}}^{x}:=\{|\theta_{*}-\hat{\theta}_{t}|>x\}$
be the event that the distance between the global parameter estimate
and its true value exceeds $x$. Similarly, let ${\cal F}_{\theta_{*},\hat{\theta}_{t}}^{k}(x):=\{|\hat{X}_{k,t}-\mu_{k}(\theta_{*})|>x\}$
be the event that the distance between the sample mean reward estimate
of arm $k$ and the true expected reward of arm $k$ exceeds $x$.
The following lemma relates these events.

\begin{lemma} \label{lemma:eventbound} For any $t \geq 2$ and given global
parameter $\theta_{*}$, we have ${\cal G}_{\theta_{*},\hat{\theta}_{t}}^{x}\subseteq \cup_{k=1}^{K}{\cal F}_{\theta_{*},\hat{\theta}_{t}}^{k}((\frac{x}{D_{1}})^{\frac{1}{\gamma_{1}}})$ with probability one.
\end{lemma}

This lemma follows from the decomposition given in Lemma \ref{lemma:gap}.
This lemma will be used to bound the probability of event ${\cal G}_{\theta_*,\hat{\theta}_{t}}^{x}$
in terms of probabilities of the events ${\cal F}_{\theta_*,\hat{\theta}_{t}}^{k}((\frac{x}{D_{1}})^{\frac{1}{\gamma_{1}}})$.
\vspace{-0.1in}
\subsection{Parameter-Free Regret Analysis}
\vspace{-0.1in}
The following theorem bounds the expected regret of the greedy policy
in one step.

\begin{theorem} \label{thm:onestepregret} Under Assumption \ref{ass:holder}, for given global parameter $\theta_{*}$, the expected one-step regret of the greedy policy is bounded by $E[r_{t}(\theta_{*})]=O(t^{-\frac{\gamma_{1}\gamma_{2}}{2}})$. 
\end{theorem}

Theorem \ref{thm:onestepregret} does not only prove that the expected
loss incurred in one step by the greedy policy goes to zero with time but also
bounds the expected loss that will be incurred at any time step $t$.\footnote{The asymptotic notation is only used for a succinct representation, to hide the constants and highlight the time dependence. This bound holds not just asymptotically but for any finite $t$.}
This is a \emph{worst-case} bound in the sense that it does not depend
on $\theta_{*}$. Using this result, we derive the parameter-free
regret bound in the next theorem.

\begin{theorem} \label{thm:par_indep} Under Assumption \ref{ass:holder},
for given global parameter $\theta$, the parameter-free
regret of the greedy policy is bounded by $E[\text{Reg}(\theta_{*},T)] = O(K^{\frac{\gamma_{1}\gamma_{2}}{2}}T^{1-\frac{\gamma_{1}\gamma_{2}}{2}})$.
\end{theorem}

Note that the parameter-free regret bound is sublinear both in terms
of the time horizon $T$ and the number of arms $K$. Moreover, it
depends on the form of the reward functions given in Assumption 1.
The Hölder exponent $\gamma_{1}$ on the inverse reward functions characterizes the informativeness of an arm
about the other arms. The informativeness of an arm $k$ can be viewed
as the information obtained about the expected rewards of the other
arms from the rewards observed from arm $k$. The informativeness
is maximized for the case when the inverse reward functions are linear
or piecewise linear, i.e., $\gamma_{1}=1$. It is increasing $\gamma_{1}$,
which results in the regret decreasing with the informativeness. On
the other hand, the Hölder exponent $\gamma_{2}$ is related to the
loss due to suboptimal arm selections, which decreases with $\gamma_{2}$.
Both of these observations follow from Lemma \ref{lemma:gap}
and \ref{lemma:onesteploss}. As a consequence, the parameter-free
regret is decreasing in both $\gamma_{1}$ and $\gamma_{2}$.

When the reward functions are linear or piecewise linear, we have
$\gamma_{1}=\gamma_{2}=1$; hence, the parameter-free regret is $O(\sqrt{T})$,
which matches with the worst-case regret bound of standard MAB algorithms
in which a linear estimator is used \cite{BubeckBianchi} and bounds given for linearly parametrized bandits \cite{Tsiklis_structured}.

\vspace{-0.15in}
\subsection{Parameter-Dependent Regret Analysis} \label{subsec:par_dep}
\vspace{-0.10in}
Although the regret bound derived in the previous section holds for
any global parameter value, it is easy to see that the performance
of the greedy policy depends on the true value of the global parameter.
For example, it is easier to identify the optimal arm in GMAB problems
with large suboptimality distance than GMAB problems which have small
suboptimality distance. In this section, we prove a regret bound that
depends on the suboptimality distance. Moreover, our regret bound
is characterized by three regimes of growth: sublinear growth followed
by logarithmic growth followed by a constant bound.

The boundaries of these regimes are defined by parameter-dependent
(problem-specific) constants. 
\begin{definition} Let $C_{1}(\Delta_{*})$
be the least integer $\tau$ such that $\tau\geq\frac{D_{1}^{\frac{2}{\gamma_{1}}}K}{2{\Delta_{*}}^{\frac{2}{\gamma_{1}}}}\log(\tau)$
and let $C_{2}(\Delta_{*})$ be the least integer $\tau$ such that
$\tau\geq\frac{D_{1}^{\frac{2}{\gamma_{1}}}K}{{\Delta_{*}}^{\frac{2}{\gamma_{1}}}}\log(\tau)$.
\end{definition} 
The constants $C_{1}(\Delta_{*})$ and $C_{2}(\Delta_{*})$
depend on the informativeness (Hölder exponent $\gamma_1$) and global parameter $\theta_{*}$.
We define the expected regret between time $T_{1}$ and $T_{2}$ for global parameter $\theta_{*}$ as 

\vspace{-0.35in}

\begin{align}
R_{\theta_{*}}(T_{1},T_{2}):=E[\text{Reg}(T_{2},\theta_{*})-\text{Reg}(T_{1},\theta_{*})].
\end{align}

\vspace{-0.08in}

The following theorem gives a three regime parameter-dependent regret
bound. 
\begin{theorem} Under Assumptions \ref{ass:holder}, the regret
of the greedy policy is bounded as follows:  If \\
 (i) $1\leq T\leq C_{1}(\Delta_{*})$, the regret is sublinear in
time, i.e., 
\vspace{-0.1in}
\begin{align}
R_{\theta_{*}}(T,0)=O(T^{1-\frac{\gamma_{1}\gamma_{2}}{2}}),\label{thm:par_dep}
\end{align}
(ii) $C_{1}(\Delta_{*})\leq T\leq C_{2}(\Delta_{*})$, the regret
is logarithmic in time, i.e., 
\vspace{-0.1in}
\begin{align}
R_{\theta_{*}}(T,C_{1}(\Delta_{*}))\leq1+2K\log(\frac{T}{C_{1}(\Delta_{*})}),
\end{align}
(iii) $T\geq C_{2}(\Delta_{*})$, the regret is bounded, i.e., 
\vspace{-0.1in}
\begin{align}
R_{\theta_{*}}(T,C_{2}(\Delta_{*}))\leq K\frac{\pi^{2}}{3}
\end{align}
\end{theorem}

\begin{corollary} The regret of the greedy policy is bounded, i.e.,
$\lim_{T\rightarrow\infty}\text{Reg}(T,\theta_{*})<\infty$. \end{corollary}

These results are obtained when  Assumption \ref{ass:holder} holds, which implies that the reward functions are invertible. We provide a counter example for a non- invertible reward function to show that bounded regret is not possible for general non-invertible reward functions. 

\textbf{Counter 	Example} : All expected arm rewards
come from a set with $K$ distinct elements. There are
$K!$ permutations of these distinct elements, and the global parameter space $\Theta$ is divided into $K!$ intervals such that the expected reward distribution of each arm in each interval
is constant and equals to the value of the element it corresponds to in one of the permutations. In order to identify
the arm rewards correctly, we have to know the permutation and hence,
the parameter value $\theta_{*}$. However, we cannot identify all the arms
correctly without playing all of them separately because an arm
can have the same expected reward in different permutations (for different parameter intervals), but at least one of the other arms will have a different expected reward in these permutations.  

In each time $t\leq T$ in each regime in Theorem \ref{thm:par_dep},
the probability of selecting a suboptimal arm is bounded by different
functions of $t$, which leads to different growth rates of the regret
bound depending on the value of $T$. For instance, when $C_{1}(\Delta_{*})\leq t\leq C_{2}(\Delta_{*})$,
the probability of selecting a suboptimal arm is in the order of $t^{-1}$;
hence, the greedy policy achieves the logarithmic regret, when $t\geq C_{2}(\Delta_{*})$,
the probability of selecting a suboptimal arm is in the order of $t^{-2}$,
which makes the probability of selecting a suboptimal arm infinitely
often zero. In conclusion, the greedy policy achieves bounded regret. Note
that a bounded regret is the striking difference between the standard
MAB algorithms \cite{lairobbinsl,A2002} and the proposed policy.

\begin{theorem} \label{thm:convergence} The sequence of arms selected
by the greedy policy converges to the optimal arm almost surely, i.e.,
$\lim_{t\rightarrow\infty}I_{t}=k^{*}(\theta_{*})$ with
probability 1. \end{theorem}

\vspace{-0.05in}

Theorem \ref{thm:convergence} implies that a suboptimal arm is selected
by greedy policy only finitely many times. In other words, there exists
a finite number such that selection of greedy policy is the optimal
arm after that number with probability $1$. This is the biggest difference
between MAB algorithms \cite{lairobbinsl,A2002} in which suboptimal
arms are selected infinitely many times and the proposed greedy policy.

Although the parameter dependent regret bound is finite, since $\lim_{\Delta_{*}\rightarrow0}C_{1}(\Delta_{*})=\infty$,
in the worst-case, this bound reduces to the parameter-free regret
bound given in Theorem \ref{thm:par_indep}.

\vspace{-0.1in}
\section{Bayesian Risk Analysis of the Greedy Policy}
\vspace{-0.1in}
In this section, assuming that global parameter is drawn from an unknown distribution $f(\theta_{*})$ on $\Theta$, we provide an analysis of the Bayesian risk, which is defined as follows: 

\vspace{-0.2in}
\begin{align}
\hspace{-0.1in}\text{Risk}(T)=E_{\theta_{*}\sim f(\theta_{*})}\left[E_{\boldsymbol{X}_{t}\sim\boldsymbol{\nu}}\left[\sum_{t=1}^{T}r_{t}(\theta)|\theta_{*}=\theta\right]\right],
\end{align}

\vspace{-0.1in}

$\boldsymbol{\nu}=\times_{k=1}^{K}\nu_{k}(\theta_{*})$ is the joint distribution
of the rewards given the parameter value is $\theta_{*}$. The Bayesian
risk is equal to the expected regret with respect to the distribution
of the global parameter $f(\theta_{*})$. Since suboptimality distance is a function of global parameter $\theta_{*}$, there is a prior distribution on the minimum sub optimality distance, which we denote as $g(\Delta_{*})$. A simple upper bound on the
Bayesian risk can be obtained by taking the expectation of the regret
bound given in Theorem \ref{thm:par_indep} with respect to $\theta_{*}$,
which gives the bound $\text{Risk}(T)=O(T^{1-\frac{\gamma_{1}\gamma_{2}}{2}})$.
Next, we will show that a tighter regret bound on the Bayesian risk
can be derived if the following assumption holds.

\begin{assumption} \label{ass:bayes_assumptions}The prior distribution
on the global parameter is such that minimum sub optimality distance $\Delta_{*}$
has a bounded density function, i.e., $g(\Delta_{*})\leq B$. One
example of this is the case when $f(\theta_{*})$ is bounded. 
\end{assumption}

Assumption \ref{ass:bayes_assumptions} is satisfied for many instances
of the GMAB problem. An example is a GMAB problem with two arms, $f(\theta_{*})\sim\textrm{Uniform}([0,1])$,
$\mu_{1}(\theta_{*})=\theta_{*}$ and $\mu_{2}(\theta_{*})=1-\theta_{*}$. For this
example we have $g(\Delta_{*})\leq2$ for $\Delta_{*}\in[0,0.5]$.

\begin{theorem} \label{thm:risk} Under Assumptions \ref{ass:holder}
and \ref{ass:bayes_assumptions}, the Bayesian risk of the greedy
policy is bounded by \\
 (i) $\text{Risk}(T)=O(\log T)$, for $\gamma_{1}\gamma_{2}=1$. \\
 (ii) $\text{Risk}(T)=O(T^{1-\gamma_{1}\gamma_{2}})$, for $\gamma_{1}\gamma_{2}<1$.
\end{theorem}

Our Bayesian risk bound for the greedy policy coincides with the Bayesian
risk bound for the linearly-parametrized MAB problem given in \cite{Tsiklis_structured}
when the arms are fully informative, i.e., $\gamma_{1}\gamma_{2}=1$.
For this case, the optimality of the Bayesian risk bound is established
in \cite{Tsiklis_structured}, in which a lower bound of $\Omega(\log T)$
is proven. Similar to the parameter-free regret bound given in Theorem
\ref{thm:par_indep}, the Bayesian risk is also decreasing with the
informativeness, and minimized for the case when the arms are fully
informative.
\vspace{-0.1in}
\section{Extension to Bandits with Group Informativeness}
\vspace{-0.1in}
Our global informativeness assumption can be relaxed to {\em group
informativeness}. When the arms are group informative, reward observations
from an arm only provides information about the rewards of the arms
that are within the same group with the original arm. Let ${\cal C}=(C_{1},\ldots,C_{D})$
be be the set of the groups, and assume that they are known by the
learner. Then, a standard MAB algorithm such as UCB1 \cite{A2002}
can be used to select the group, while the greedy policy can be used
to select among the arms within a group. In this way, we can exploit
the informativeness among the arms within a group and find the group
to which the best arm belongs by a standard MAB algorithm. In this
way it is possible to achieve bounded regret within each group. However,
in order to identify the group to which the optimal arm belongs, each
groups should be selected at least logarithmically many times by the
standard MAB algorithm. As a result, the combination of two algorithms
yields a regret bound of $O(D\log T)$ which depends on the number
of groups instead of the number of arms. The formal derivation of this result is left as future work.
\vspace{-0.2in}
\section{Conclusion}
\vspace{-0.15in}
In this paper we introduce a new class of MAB problems called global
multi-armed bandits. This general class of GMAB problems encompasses
the previously introduced linearly-parametrized bandits as a special
case. We proved that the regret for the GMABs has three regimes, which
we characterized for the regret bound, and showed that the parameter-dependent
regret is bounded, i.e., it is asymptotically finite. In addition to
this, we also proved a parameter-free regret bound and a Bayesian
risk bound, both of which grow sublinearly over time, where the rate
of growth depends on the informativeness of the arms. Future work
includes extension of global informativeness to group informativeness,
and a foresighted MAB problem, where the arm selection is based on
a foresighted policy that explores the arms according to their level
of informativeness rather than the greedy policy.
\vspace{-0.2in}
\section{Proofs}
\vspace{-0.1in}
In this section, we provide the proofs of theorems. The proofs of
lemmas are given in the supplementary material. Let $\boldsymbol{w}(t) := (w_1(t), \ldots, w_K(t))$ be the vector of weights and $\boldsymbol{N}(t) := (N_1(t), \ldots, N_k(t))$ be the vector of counters at time $t$. We have $\boldsymbol{w}(t) = \frac{1}{t} \boldsymbol{N}(t)$. Since $\boldsymbol{N}(t)$ depends on the history, they are both random variables depending on the obtained rewards.
\vspace{-0.15in}
\subsection{Proof of Theorem 1}
\vspace{-0.1in}
By lemma \ref{lemma:onesteploss} and Jensen's inequality, we have
\begin{align}
E[r_{t}(\theta_{*})]\leq2D_{2}E[|\theta_{*}-\hat{\theta}_{t}|]^{\gamma_{2}}. \label{eq:r_t}
\end{align}
By using Lemma \ref{lemma:gap} and Jensen's inequality, we have 
\vspace{-0.1in}
\begin{align}
&E[|\theta_{*}-\hat{\theta}_{t}|] \leq \notag \\
&D_1 E[\sum_{k=1}^{K} w_k(t) E[|\hat{X}_{k,t} -\mu_k(\theta_{*})|\;| \boldsymbol{w}(t)]^{\gamma_1}], \label{eq:gap}
\end{align}
, where $E[\cdot | \cdot]$ denotes the conditional expectation. Note that $\hat{X}_{k,t}=\frac{\sum_{x\in{\cal X}_{k,t}} x}{N_{k}(t)}$
and $E_{x\sim\nu_{k}(\theta_{*})}[x]=\mu_{k}(\theta_{*})$. Therefore, we can bound
$E[|\hat{X}_{k,t}-\mu_{k}(\theta_{*})|\;| \boldsymbol{w}(t)]$ for each $k\in{\cal K}$ using
Chernoff- Hoeffding inequality. For each $k \in {\cal K}$, we have
\begin{align}
 & E[|\hat{X}_{k,t}-\mu_{k}(\theta_{*})|\;| \boldsymbol{w}(t)] \notag \\ 
 &=\int_{x=0}^{1}\!\text{Pr}(|\hat{X}_{k,t}-\mu_{k}(\theta_{*})|>x | \boldsymbol{w}(t))\,\mathrm{d}x \notag \\
 &\leq \int_{x=0}^{\infty}\!2\exp(-2x^{2}N_{k}(t))\,\mathrm{d}x \leq\sqrt{\frac{\pi}{2N_{k}(t)}} , \label{eq:chernoff1}
\end{align}
, where $N_k(t) = t w_k(t)$ is a random variable. The first inequality is a result of the Chernoff-Hoeffding bound. Combining (\ref{eq:gap}) and (\ref{eq:chernoff1}), we get
\vspace{-0.1in}
\begin{align}
E[|\theta_{*} - \hat{\theta}_t|] \leq 2D_{1}(\frac{\pi}{2})^{\frac{\gamma_1}{2}}\frac{1}{t^{\frac{\gamma_1}{2}}}  E[\sum_{k=1}^K {w_k(t)}^{1- \frac{\gamma_1}{2}}]. \label{eq:gap2}
\end{align}
Since $w_k(t) \leq 1$ for all $k \in {\cal K}$, and $\sum_{k=1}^K w_k(t) =1$ for any possible $\boldsymbol{w}(t)$, we have $E[\sum_{k=1}^{K} w_{k}(t)^{1-\frac{\gamma_{1}}{2}}]\leq K^{\frac{\gamma_{1}}{2}}$. Then, combining (\ref{eq:r_t}) and (\ref{eq:gap2}), we have
\vspace{-0.10in}
\begin{align}
E[r_{t}(\theta_{*})]\leq 2 D_{1}^{\gamma_{2}}D_{2}\frac{\pi}{2}^{\frac{\gamma_{1}\gamma_{2}}{2}} K^{\frac{\gamma_{1}\gamma_{2}}{2}}\frac{1}{t^{\frac{\gamma_{1}\gamma_{2}}{2}}} .
\end{align}
\vspace{-0.3in}

\subsection{Proof of Theorem 2}
\vspace{-0.1in}
The bound is consequence of Theorem \ref{thm:onestepregret}
and inequality given in \cite{bound}, i.e.,
\begin{align}
 E[\text{Reg}(\theta_{*},T)] \leq 1 + \frac{2D_{1}^{\gamma_{2}}D_{2}\frac{\pi}{2}^{\frac{\gamma_{1}\gamma_{2}}{2}}K^{\frac{\gamma_{1}\gamma_{2}}{2}}}{1-\frac{\gamma_{1}\gamma_{2}}{2}}(1+ T^{1-\frac{\gamma_{1}\gamma_{2}}{2}}). \notag
\end{align}

\vspace{-0.20in}
\subsection{Proof of Theorem 3}
\vspace{-0.1in}
We need to bound the probability of the event that ${I_{t}\neq k^{*}(\theta_*)}$.
Since at time $t$, the arm with the highest $\mu_{k}(\hat{\theta}_{t})$ is selected
by the greedy policy, $\hat{\theta}_{t}$ should lie in $\Theta\setminus\Theta_{k^{*}(\theta_{*})}$
for greedy policy to select a suboptimal arm. Therefore, we can write,
\begin{align}
 \{{I_{t}\neq k^{*}(\theta_{*})}\}=\{{\hat{\theta}_{t}\in\Theta\setminus\Theta_{k^{*}(\theta_{*})}}\}\subseteq{{\cal G}_{\theta_{*},\hat{\theta}_{t}}^{\Delta_{*}}} . \label{eq:evbound}
\end{align}
By Lemma \ref{lemma:eventbound} and (\ref{eq:evbound}), we have
\begin{align}
 & \Pr(I_{t}\neq k^{*}(\theta_*))\leq\sum_{k=1}^{K}E[E[I({\cal F}_{\theta_*,\hat{\theta}_{t}}^{k}((\frac{x}{D_{1}})^{\frac{1}{\gamma_{1}}}))| \boldsymbol{N}(t)]] \notag\\
 & =\sum_{k=1}^{K}E[Pr({\cal F}_{\theta_*,\hat{\theta}_{t}}^{k}((\frac{x}{D_{1}})^{\frac{1}{\gamma_{1}}})| \boldsymbol{N}(t))] \notag\\
 & \leq\sum_{k=1}^{K}2E[\exp(-2(\frac{\Delta_{*}}{D_{1}})^{\frac{2}{\gamma_{1}}}N_{k}(t))]\notag\\
 & \leq 2K\exp(-2(\frac{\Delta_{*}}{D_{1}})^{\frac{2}{\gamma_{1}}}\frac{t}{K}). \label{eqn:probimportant}
\end{align}
, where the first inequality is followed by union bound and second inequality
is obtained by using the Chernoff-Hoeffding bound. The last inequality is obtained by
using the worst-case selection processes $N_{k}(t)=\frac{t}{K}$.
We have $\Pr(I_{t}\neq k^{*}(\theta_{*}))\leq\frac{1}{t}$ for $t>C_{1}(\Delta_{*})$
and $\Pr(I_{t}\neq k^{*}(\theta_{*}))\leq\frac{1}{t^{2}}$ for $t>C_{2}(\Delta_{*})$.
The bound in the first regime is the result of Theorem \ref{thm:par_indep}.
The bound in the second and third regimes is obtained by summing the probability given in (\ref{eqn:probimportant})
from $C_{1}(\Delta_{*})$ to $T$ and $C_{2}(\Delta_{*})$ to $T$,
respectively. 

\subsection{Proof of Theorem 4}

Let $(\Omega,{\cal F},P)$ denote probability space, where $\Omega$ is the sample set and ${\cal F}$ is the $\sigma$-algebra that the probability measure $P$ is defined on. Let $\omega \in \Omega$ denote a sample path. We will prove that there exists event $N\in{\cal F}$ such that $P(N)=0$ and if $\omega \in N^{c}$, then $\lim_{t\rightarrow\infty}I_{t}(\omega)=k^{*}(\theta_{*})$.
Define the event ${\cal E}_{t} :=\{I_{t}\neq k^{*}(\theta_{*})\}$.
We show in the proof of Theorem \ref{thm:par_dep} that $\sum_{t=1}^{T}P({\cal E}_{t})<\infty$.
By Borel-Cantelli lemma, we have 
\vspace{-0.1in}
\begin{align}
P({\cal E}_{t}\text{ infintely often})=P(\limsup_{t\rightarrow\infty}{\cal E}_{t})=0.
\end{align}

\vspace{-0.15in}

Define $N :=\limsup_{t\rightarrow\infty}{\cal E}_{t}$, where $P(N)=0$.
We have, 

\vspace{-0.3in}

\begin{align}
N^{\text{c}}=\liminf_{t\rightarrow\infty}{\cal E}_{t}^{\text{c}},
\end{align}

\vspace{-0.15in}

, where $P(N^{\text{c}})=1-P(N)=1$, which means that $I_{t}=k^{*}(\theta_*)$
for all $t$ except for a finite number.

\subsection{Proof of Theorem 5}

\begin{proof} The one step loss due to suboptimal arm selection with
global parameter estimate $\hat{\theta}_{t}$ is given in Lemma \ref{lemma:onesteploss}.
Recall that we have 
\begin{align}
\{I_{t}\neq k^{*}(\theta_{*})\}\subseteq\{|\theta_{*}-\hat{\theta}_{t}|>\Delta_{*}\} . \notag
\end{align}
Let $Y_{\theta_{*},\hat{\theta}_{t}} :=|\theta_{*}-\hat{\theta}_{t}|$.
Then, we have
\begin{align}
& \text{Risk}(T) \notag \\
 &\leq2D_{2}\sum_{t=1}^{T}E_{\theta_{*}\sim f(\theta)}[E_{\boldsymbol{X}\sim\boldsymbol{\nu}}[Y_{\theta_{*},\hat{\theta}_{t}}^{\gamma_{2}}I(Y_{\theta_{*},\hat{\theta}_{t}}>\Delta_{*})]]\notag\\
 & \leq2D_{2}\sum_{t=1}^{T}E_{\theta_{*} \sim f(\theta)}[E_{\boldsymbol{X}\sim\boldsymbol{\nu}}[Y_{\theta_{*},\hat{\theta}_{t}}I(Y_{\theta_{*},\hat{\theta}_{t}}>\Delta_{*})]]^{\gamma_{2}} \notag ,
\end{align}
, where $I(.)$ is the indicator function which is $1$ if the statement
is true and zero otherwise. The first inequality followed by Lemma
\ref{lemma:gap}. The second inequality is by Jensen's inequality
and the fact that $I(.)=I^{\gamma}(.)$ for any $\gamma>0$. We now
focus on the expectation expression for some arbitrary $t$. Let $f(\theta)$
denote the density function of global parameter. 
\begin{align}
 & E_{\theta_{*}\sim f(\theta)}[E_{\boldsymbol{X}\sim\boldsymbol{\nu}}[Y_{\theta_{*},\hat{\theta}_{t}}I(Y_{\theta_{*},\hat{\theta}_{t}}>\Delta_{*})]]\notag\\
 & =\int_{\theta_{*}=0}^{1}\! f(\theta_{*})\int_{x=0}^{\infty}\!\Pr(Y_{\theta_{*},\hat{\theta}_{t}}I(Y_{\theta_{*},\hat{\theta}_{t}}>\Delta_{*})\geq x)\,\mathrm{d}x\,\mathrm{d}{\theta}\notag\\
 & =\int_{\theta_{*}=0}^{1}\! f(\theta_{*})\int_{x=\Delta_{*}}^{\infty}\!\Pr(Y_{\theta_{*},\hat{\theta}_{t}}\geq x)\,\mathrm{d}x\,\mathrm{d}{\theta}\notag\\
 & =\int_{\Delta=0}^{1}\! g(\Delta)\int_{x=\Delta}^{\infty}\!\Pr(Y_{\theta_{*},\hat{\theta}_{t}}\geq x)\,\mathrm{d}x\,\mathrm{d}{\Delta},  \notag
\end{align}
, where the last equation is followed by change of variables in integral. Note that we have by Theorem \ref{thm:par_dep} 
\begin{align}
\Pr(Y_{\theta_{*},\hat{\theta}_{t}}\geq x)\leq2K\exp(-2x^{\frac{2}{\gamma_{1}}}D_{1}^{-\frac{2}{\gamma_{1}}}\frac{t}{K}). \notag
\end{align}
Then, we have
\begin{align}
 & E_{\theta\sim\nu_{\theta}}[E_{\boldsymbol{X}\sim\boldsymbol{\nu}}[Y_{\theta_{*},\hat{\theta}_{t}}I(Y_{\theta_{*},\hat{\theta}_{t}}>\Delta_{*})]]\notag\\
 & \leq2KB\int_{\Delta=0}^{1}\!\exp(-2{\Delta}^{\frac{2}{\gamma_{1}}}D_{1}^{-\frac{2}{\gamma_{1}}}\frac{t}{K})\,\mathrm{d}{\Delta}\notag\\
 & \int_{y=0}^{\infty}\!\exp(-2y^{\frac{2}{\gamma_{1}}}D_{1}^{-\frac{2}{\gamma_{1}}}\frac{t}{K})\,\mathrm{d}y\notag\\
 & =2KB(\frac{\gamma_{1}}{2}2^{-\frac{\gamma_{1}}{2}}D_{1}K^{\frac{\gamma_{1}}{2}} \Gamma(\frac{\gamma_1}{2}))^{2}t^{-\gamma_{1}} , \notag
\end{align}
, where the inequality follows from the change of variable $y=x-\Delta$ and
then the fact that $(y+\Delta)^{\frac{2}{\gamma_{1}}}\geq y^{\frac{2}{\gamma_{1}}}+\Delta^{^{\frac{2}{\gamma_{1}}}}$
since $\frac{2}{\gamma_{1}}\geq1$. By summing these from $1$ to
$T$, we get

\vspace{-0.2in}

\[
\text{Risk}(T)\leq\left\{ \begin{array}{lr}
1 + A(1+2\log T) & \text{if }\gamma_{1}\gamma_{2}=1\\
1+ A(1+ \frac{1}{1-\gamma_{1} \gamma_{2}} T^{1-\gamma_{1}\gamma_{2}}) & \text{if }\gamma_{1}\gamma_{2}<1
\end{array}\right.
\]

, where $A=2D_{2}(\frac{B\gamma_{1}^{2}D_{1}^{2}K^{1+\gamma_{1}}}{2^{1+\gamma_{1}}}\Gamma^{2}(\frac{\gamma_{1}}{2}))$.
\end{proof}

\section{Appendix}
\begin{lemma} \label{prop:non_zero} Given any $\theta_{*}\in\Theta$,
there exists a constant $\epsilon_{\theta_{*}}=\delta_{\min}(\theta_{*})^{1/\gamma_{2}}/(2D_{2})^{1/\gamma_{2}}$,
where $D_{2}$ and $\gamma_{2}$ are the constants given in Assumption
1 such that $\Delta_{\min}(\theta_{*})\geq\epsilon_{\theta_{*}}.$ In other
words, the minimum suboptimality distance is always positive. 
\end{lemma}

\begin{proof} For any suboptimal arm $k\in{\cal K}-k^{*}(\theta)$,
we have $\mu_{k^{*}(\theta)}(\theta)-\mu_{k}(\theta)\geq\delta_{\min}(\theta)>0.$
We also know that $\mu_{k}(\theta')\geq\mu_{k^{*}(\theta)}(\theta')$
for all $\theta'\in\Theta_{k}$. Hence for any $\theta'\in\Theta_{k}$
at least one of the following should hold: (i) $\mu_{k}(\theta')\geq\mu_{k}(\theta)-\delta_{\min}(\theta)/2$,
(ii) $\mu_{k^{*}(\theta)}(\theta')\leq\mu_{k^{*}(\theta)}(\theta)+\delta_{\min}(\theta)/2$.
If both of the below does not hold, then we must have $\mu_{k}(\theta')<\mu_{k^{*}(\theta)}(\theta')$,
which is false. This implies that we either have $\mu_{k}(\theta)-\mu_{k}(\theta')\leq\delta_{\min}(\theta)/2$
or $\mu_{k^{*}(\theta)}(\theta)-\mu_{k^{*}(\theta)}(\theta')\geq - \delta_{\min}(\theta)/2$,
or both. Recall that from Assumption 1 we have $|\theta-\theta'|\geq|\mu_{k}(\theta)-\mu_{k}(\theta')|^{1/\gamma_{2}}/D_{2}^{1/\gamma_{2}}$.
This implies that $|\theta-\theta'|\geq\epsilon_{\theta}$ for all
$\theta'\in\Theta_{k}$. 
\end{proof}

\begin{lemma} \label{lemma:gap} Consider a run of the greedy policy
until time $t$. Then, the following relation between $\hat{\theta}_{t}$
and $\theta_{*}$ holds with probability one: $|\hat{\theta}_{t}-\theta_{*}|\leq\sum_{k=1}^{K}w_{k}(t)D_{1}|\hat{X}_{k,t}-\mu_{k}(\theta_{*})|^{\gamma_{1}}$
\end{lemma}

\begin{proof}
\begin{align}
&|\theta_{*} - \hat{\theta}_t| = |\sum_{k=1}^{K} w_k(t)\hat{\theta}_{k,t} -\theta_{*}| \notag \\ 
& = \sum_{k=1}^{K} w_k(t)|\theta_{*} - \hat{\theta}_{k,t}| \notag \\
& = \sum_{k=1}^{K} w_k(t)|\mu^{-1}_k(\hat{X}_{k,t}) - \mu^{-1}_k(\mu_k(\theta_{*}))| \notag \\ 
& \leq \sum_{k=1}^{K} w_k(t)D_1|\hat{X}_{k,t} - \mu_k(\theta_{*})|^{\gamma_1} 
\end{align}
, where last inequality followed by Assumption 1.
\end{proof}

\begin{lemma} \label{lemma:onesteploss} For given global parameter $\theta_*$, the one step regret of the
greedy policy is bounded by $r_{t}(\theta_{*})=\mu^{*}(\theta_{*})-\mu_{I_{t}}(\theta_{*}) \leq 2D_{2}|\theta_{*}-\hat{\theta}_{t}|^{\gamma_{2}}$ with probability one, where $I_{t}$ is the arm selected by the greedy policy at time $t\geq2$. \end{lemma}

\begin{proof}
Note that $I_t \in \argmax_{k \in {\cal K}} \mu_k(\hat{\theta}_t)$. Therefore, we have 
\begin{align}
\mu_{I_t}(\hat{\theta}_t) - \mu_{k^{*}(\theta_{*})}(\hat{\theta}_t) \geq 0 \label{eq:add} . 
\end{align}

We have $\mu^{*}(\theta_{*}) = \mu_{k^{*}(\theta_{*})}(\theta_{*})$. Then, we can bound
\begin{align}
&\mu^{*}(\theta_{*})-\mu_{I_{t}}(\theta_{*}) \notag \\ 
&= \mu_{k^{*}(\theta_{*})}(\theta_{*}) - \mu_{I_{t}}(\theta_{*}) \notag \\ 
& \leq  \mu_{k^{*}(\theta_{*})}(\theta_{*}) - \mu_{I_{t}}(\theta_{*}) + \mu_{I_t}(\hat{\theta}_t) - \mu_{k^{*}(\theta_{*})}(\hat{\theta}_t) \notag \\
& = \mu_{k^{*}(\theta_{*})}(\theta_{*})  - \mu_{k^{*}(\theta_{*})}(\hat{\theta}_t) + \mu_{I_t}(\hat{\theta}_t)- \mu_{I_{t}}(\theta_{*})  \notag \\
& \leq 2D_2|\theta_{*} -\hat{\theta}_t|^{\gamma_2}
\end{align}
, where the first inequality followed by inequality \ref{eq:add} and second inequality by Assumption 1. 
\end{proof}

\begin{lemma} \label{lemma:eventbound} For any $t \geq 2$ and given global
parameter $\theta_{*}$, we have ${\cal G}_{\theta_{*},\hat{\theta}_{t}}^{x}\subseteq \cup_{k=1}^{K}{\cal F}_{\theta_{*},\hat{\theta}_{t}}^{k}((\frac{x}{D_{1}})^{\frac{1}{\gamma_{1}}})$ with probability one.
\end{lemma}

\begin{proof}
\begin{align}
&\{ |\theta_{*} - \hat{\theta}_t| \geq x \} \notag \\ 
&\subseteq \{ \sum_{k=1}^{K} w_k(t) D_1 |\hat{X}_{k,t} - \mu_k(\theta_{*})| \geq x\} \notag \\ 
& \subseteq \cup_{k=1}^{K}  \{ w_k(t) D_1 |\hat{X}_{k,t} - \mu_k(\theta_{*})| \geq w_k(t) x\} \notag \\ 
& =  \cup_{k=1}^{K} \{ |\hat{X}_{k,t} - \mu_k(\theta_{*})| \geq (\frac{x}{D_1})^{\frac{1}{\gamma_1}} \}
\end{align}
, where the first inequality followed by Lemma \ref{lemma:gap} and second inequality by the fact that $\sum_{k=1}^K w_k(t) =1$.
\end{proof}

 \bibliographystyle{IEEE}
\bibliography{aistats}

\end{document}